\documentclass[a4paper,10pt,twoside]{article}

\usepackage{times}  
\usepackage{geometry}
\usepackage{macros}
\usepackage{authblk}
\RequirePackage{algorithm}
\RequirePackage{algorithmic}
\usepackage{caption}
\usepackage[numbers]{natbib}
\begin{document}

	\title{Estimation of Spectral Risk Measures}
	\author[1]{Ajay Kumar Pandey\thanks{ajaykp@cse.iitm.ac.in}}
	\author[1]{Prashanth L. A.\thanks{prashla@cse.iitm.ac.in}}
	\author[2]{Sanjay P. Bhat\thanks{sanjay.bhat@tcs.com}}	
	\affil[1]{\small Department of Computer Science and Engineering, Indian Institute of Technology Madras, Chennai}
	\affil[2]{\small Tata Consultancy Services Limited, Hyderabad}
	\date{}
	\maketitle
	
	\begin{abstract}
		We consider the problem of estimating a spectral risk measure $(\textrm{SRM})$ from i.i.d. samples, and propose a novel method that is based on numerical integration. We show that our $\textrm{SRM}$ estimate concentrates exponentially, when the underlying distribution has bounded support. Further, we also consider the case when the underlying distribution is either Gaussian or exponential, and derive a concentration bound for our estimation scheme. We validate the theoretical findings on a synthetic setup, and in a vehicular traffic routing application.
	\end{abstract}
	
	\section{Introduction}
	\label{introduction}
	In the context of risk-sensitive optimization, Conditional value-at-risk ($\textrm{CVaR}$) is a popular measure.
	CVaR is a conditional expectation of a random variable (r.v.) that usually models the losses in an application (e.g. finance), where the conditioning is based on value-at-risk $(\textrm{VaR})$. The latter denotes the maximum loss that could be incurred, with high probability. The advantage of employing $\textrm{CVaR}$ instead of $\textrm{VaR}$ in a risk-sensitive optimization setting is that $\textrm{CVaR}$ is a coherent risk measure \cite{artzner1999coherent}, while $\textrm{VaR}$ is not, as it violates the sub-additivity assumption. 
	
	Spectral risk measures $(\textrm{SRM})$ are a generalization of CVaR, and are defined as follows:
	\begin{align}
	\mathrm{S}(X) = \int_{0}^{1}\varphi(\beta)\mathrm{V}_{\beta}(X)\,d\beta.\label{def:srm}
	\end{align}
	In the equation above, $\varphi(\cdot)$ is a risk-aversion function, which can be chosen to ensure that SRM is a coherent risk measure \cite{acerbi2002spectral} and $V_\beta$ is the $\beta$-quantile of the distribution of the r.v. $X$.
	In particular, a non-negative, increasing $\varphi$ that integrates to $1$ is sufficient for ensuring coherence. 
	SRM can be seen as a weighted average of the quantiles (VaR) of the underlying distribution. Moreover, CVaR can be recovered by setting $\varphi(\beta) = \frac{1}{1-\alpha}\indic{\beta>\alpha},\, \alpha \in (0,1)$. The latter choice translates to an equal weight for all tail-loss VaR values. In contrast, SRM can model a user's risk aversion better, since the function $\varphi$ can be chosen such that higher losses receive a higher weight, or at least, the same weight as lower losses \cite{dowd2006after}.
	
	In this paper, we consider the problem of estimating SRM  of a random variable (r.v.), given independent and identically distributed (i.i.d.) samples from the underlying distribution. 
	In this context, our contributions are as follows:
	First, we provide a natural estimation scheme for SRM that uses the empirical distribution function (EDF) to estimate quantiles, together with a trapezoidal rule-based approximation to the integral in \eqref{def:srm}. 
	Second, we provide concentration bounds for our proposed SRM estimate for the following two cases: first, when the underlying distribution is assumed to have bounded support; and second, when the distribution is either Gaussian or exponential. To the best of our knowledge, no concentration bounds are available for SRM estimation.
	Third,
	we perform simulation experiments to show the efficacy of our proposed SRM estimation scheme. In particular, we consider a synthetic setup, and show that our  scheme provides accurate estimates of SRM. Next, we incorporate our SRM estimation scheme in the inner loop of the successive rejects (SR) algorithm \cite{audibert2010best}, which is a popular algorithm in the best arm identification framework for multi-armed bandits. We test the resulting SR algorithm variant in a vehicular traffic routing application using SUMO traffic simulator \cite{behrisch2011sumo}. The application is motivated by the fact that, in practice, human road users may not always prefer the route with lowest mean delay. Instead, a route that minimized worst-case delay, while doing reasonably well on the average, is preferable, and such a preference can be encoded into the risk aversion function $\varphi(\cdot)$ in \eqref{def:srm}.
	
	To the best of our knowledge, concentration bounds are not available for the SRM estimation. However, the bounds that we derive for SRM estimation could be specialized to the case of CVaR. 
	In \citep{brown2007large,wang2010deviation} concentration bounds for the classic CVaR estimator are derived. Our bound, using a different estimator, exhibits a similar rate of exponential convergence around true CVaR. For the case of distributions with unbounded support, concentration bounds for empirical CVaR have been derived recently in \citep{kolla2019concentration,kolla2019cvarBandits,bhat2019improved}. In \citep{kolla2019concentration} (resp. \citep{kolla2019cvarBandits,bhat2019improved}), the authors derive an one-sided concentration bound (resp. two-sided bounds), when the underlying distributions are either sub-Gaussian or sub-exponential \citep{wainwright2019high}.  In comparison to \citep{kolla2019concentration}, we derive two-sided concentration bounds for the special case of Gaussian and exponential distributions. Moreover, the bounds that we derive show an improved dependence on the number of samples, say $n$, and accuracy, say $\epsilon$, when compared to the corresponding bounds in \citep{kolla2019cvarBandits}. More precisely, the probability that the CVaR estimate is more than an $\epsilon$ away from the true CVaR is bounded above by $c_1 \exp\left(-c_2 n \epsilon^2\right)$, for large enough $n$ and some universal constants $c_1,c_2$, in our bound. On the other hand, the corresponding tail bound in \citep{kolla2019cvarBandits} is $c_1 \exp\left(-c_2 \sqrt{n} \min(\epsilon,\epsilon^2)\right)$ for the sub-Gaussian case, and $c_1 \exp\left(-c_2 n^{1/3} \min(\epsilon^{2/3},\epsilon^2)\right)$ for the sub-exponential case. Finally, in comparison to a recent result in \cite{bhat2019improved}, our bound exhibits exponential concentration
	, while the corresponding bound in \citep{bhat2019improved} shows a polynomial decay for $\epsilon>1$. 
	CVaR-based models have been explored in different contexts, for instance, in a bandit application  \citep{galichet2013exploration}, in a portfolio optimization problem \citep{krokhmal2002portfolio}, and in a general risk management setting \citep{mulvey2006applying}. In the simulation experiments, we consider each of these applications, and show the efficacy of our proposed estimation scheme in each application context.
	
	The rest of the paper is organized as follows: Section \ref{sec:background} introduces $\textrm{VaR}$, $\textrm{CVaR}$, $\textrm{SRM}$ and their estimators from i.i.d. samples,
	Section \ref{sec:srm-est-bdd} presents our estimate of $\textrm{SRM}$, together with concentration bounds for the case when the underlying distribution has bounded support. Section \ref{sec:srm-est-unbdd} presents a truncated SRM estimation scheme, and concentration bounds for the case when the underlying distribution is either Gaussian or exponential.  Section \ref{sec:expts} presents the simulation experiments, Section \ref{sec:proofs} provides the proofs of the concentration bounds in Sections \ref{sec:srm-est-bdd}--\ref{sec:srm-est-unbdd}, and finally, Section~\ref{sec:conclusions} concludes the paper.
	
	\section{Preliminaries}
	\label{sec:background}
	For a r.v. $X$,  VaR $\mathrm{V}_\beta(X)$ and CVaR $\mathrm{C}_\beta(X)$ at the level $\beta$, $\beta \in (0,1)$, are defined as follows:
	\begin{align}
	\mathrm{V}_\beta(X) := \inf \{c: P(X\le c)\ge \beta \}, \quad 
	\mathrm{C}_\beta(X) & := \mathrm{V}_\beta(X) \,\, + \frac{1}{1-\beta}\mathbb{E}[X - \mathrm{V}_\beta(X)]^{+}, \label{def:var-cvar}
	\end{align}
	where $[x]^{+} = \max(0,x)$ for a real number $x$.
	$\mathrm{V}_\beta(X)$ can be interpreted as the minimum loss that will not be exceeded with probability $\beta$. Note that, if $X$ has a continuous and strictly increasing cumulative distribution function (CDF) $F$, then $\mathrm{V}_\beta(X)$ is a
	solution to the following:
	\[	\mathbb{P}[X\le\xi] = \beta, \,\, \text{i.e., } \mathrm{V}_\beta(X) = F^{-1}(\beta).\]
	Further,
	$\mathrm{C}_\beta(X)$ can be interpreted as the expected loss, conditional on the event that the loss exceeds $\mathrm{V}_\beta(X)$, i.e., 
	$\mathrm{C}_{\beta}(X) = \mathbb{E}[X|X\ge \mathrm{V}_\beta(X)].$
	
	Let $X_i$, $i = 1, \ldots , n$ denote i.i.d. samples from the distribution of $X$. Then, the estimate of $\mathrm{V}_\beta(X)$, denoted by $\widehat{\mathrm{V}}_{n, \beta}$, is formed as follows \cite{serfling2009approximation}:
	\begin{align}
	\widehat{\mathrm{V}}_{n,\beta} &= \widehat{F}_{n}^{-1}(\beta)  = \inf\{x:\widehat{F}_n(x)\ge \beta\} \label{eq:var-est},
	\end{align}
	where $\widehat{F}_{n}(x) = \frac{1}{n}\sum_{i=1}^{n}\mathbb{I}[{X_i \le x}]$ is the EDF of $X$. Letting $X_{(1)}, \ldots, X_{(n)}$ denote the order statistics, i.e., $X_{(1)} \le X_{(2)} \le \dots \le X_{(n)}$, we have $\widehat{\mathrm{V}}_{n,\beta} = X_{(\lceil n\beta \rceil)}$.
	
	
	\section{Distributions with bounded support}
	\label{sec:srm-est-bdd}
	\subsection{Estimation scheme}
	We estimate $\mathrm{S}(X)$, given i.i.d. samples $X_1, \dots, X_n$ from the distribution of $X$, by approximating the integral in SRM definition \eqref{def:srm}. 
	Notice that the integrand $V_\beta(X)$ in \eqref{def:srm} has to be estimated using the samples. 
	Recall that $\widehat{\mathrm{V}}_{n,\beta}$ is the estimate of $\mathrm{V}_{\beta}(X)$, given by \eqref{eq:var-est}. 
	We use the weighted VaR estimate to form a discrete sum to approximate the integral, an idea motivated by the trapezoidal rule \citep{cruz2003elementary}.
	The estimate $\widehat{\mathrm{S}}_{n, m}$ of $\mathrm{S}(X)$ is formed as follows:
	\begin{equation}
	\label{eq:srm-est}
	\widehat{\mathrm{S}}_{\,n,m} = \sum_{k = 1}^{m}\frac{\varphi(\beta_{k-1})\widehat{\mathrm{V}}_{n,\beta_{k-1}} + \varphi(\beta_{k})\widehat{\mathrm{V}}_{n,\beta_{k}}}{2}\Delta\beta.
	\end{equation}
	In the equation above, $\{\beta_k\}_{k = 0}^{m}$ is a partition of $[0, 1]$ such that $\beta_0=0$ and $\beta_{k} = \beta_{k-1} + \Delta\beta$, where $\Delta\beta = 1/m$ is the length of each sub-interval. 
	
	In the next section, we present concentration bounds for the estimator presented above, assuming that the underlying distribution has bounded support.
	
	\subsection{Concentration bounds}
	\label{sec:conc-bounds}
	For notational convenience, we shall use ${\mathrm{V}}_\beta$ and $\mathrm{S}$ to denote ${\mathrm{V}}_{\beta}(X)$ and $\mathrm{S}(X)$, for any $\beta\in (0,1)$. 
	
	For all the results presented below, we let $\widehat{ \mathrm{S}}_{\,n,m}$ denote the SRM estimate formed from $n$ i.i.d. samples of $X$ and with m sub-intervals, using \eqref{eq:srm-est}.     
	Let $F$ and $f$ denote the distribution and density of $X$, respectively. 
	
	For the sake of analysis, we make one of the following assumptions:\\[0.5ex]
	\noindent \textbf{(A1)} Let $\varphi(\beta)$ be a risk-aversion function such that $\mid\varphi(\beta)\mid\le C_{1}$ and $\mid{\varphi}^\prime(\beta)\mid\le C_{2},  \,\, \forall \beta \in [0,1]$.\\[0.5ex]
	\noindent\textbf{(A1$^{\prime}$)} The conditions of (A1) hold. In addition,  $\mid{\varphi}^{\prime\prime}(\beta)\mid\le C_{3}, \,\,\forall \beta \in [0,1]$.
	
	\begin{theorem}[\textbf{SRM concentration: bounded case}]
		\label{thm:srm-conc-bdd}
		Let the r.v. $X$ be continuous and $X \le B$ a.s. Fix $\epsilon>0$.
		
		\noindent \textbf{(i)} Assume (A1) holds and $f(x) \ge 1/\delta_1 > 0 $,
		$x \le B$.
		If $\mid B\, C_{2}  + \delta_1\, C_{1}\mid \le K_1 $, and $m \ge \frac{K_1}{2\epsilon}$, then 
		\begin{align}
		&\mathbb{P}\left( \left| \mathrm{S} - \widehat{\mathrm{S}}_{\,n,m} \right|>\epsilon\right) \le
		\frac{2 K_1}{ \epsilon}\exp\left(\frac{ -n\,c\,\epsilon^2}{2C^2_{1}}\right),\label{eq:cvar-bd-1}
		\end{align}
		where $c = \min\{c_0, c_1, \dots , c_m\}$ and $c_k, k \in \{0, \dots, m\}$, is a constant that depends on the value of the density $f$ of the r.v. $X$ in a neighborhood of $\textrm{V}_{\beta_{k}}$, with $\beta_k$ as in \eqref{eq:srm-est}.\\\\
		\noindent \textbf{(ii)} Assume (A1$^{\prime}$) holds and $\mid f^{'}(x)\mid/f(x)^3 \le \delta_2$,
		$ x \le B$.
		If $\mid B\, C_{3}  + 2\,\delta_1\, C_{2}  + \delta_2\, C_{1}\mid \le K_2 $, and $m \ge \sqrt{\frac{K_2}{6\epsilon}}$, then 
		\begin{align}
		&{\mathbb{P}}\left( \left| \mathrm{S} - \widehat{\mathrm{S}}_{\,n,m} \right|>\epsilon\right)\le \sqrt{\frac{8K_2}{3\epsilon}}\exp\left(\frac{ -n\,c\,\epsilon^2}{2C^2_{1}}\right),\label{eq:cvar-bd-2} \end{align}
		where $c$ is as in the case above.
	\end{theorem}
	\begin{proof}
		See Section \ref{sec:proof-srm-conc-bdd}.
	\end{proof}	
	For small values of $\epsilon$, the bound in \eqref{eq:cvar-bd-2} is better than that in \eqref{eq:cvar-bd-1}. However, the bound in \eqref{eq:cvar-bd-1} is derived under weaker assumptions on the r.v. $X$ and the risk-aversion function $\varphi$, as compared to the bound in \eqref{eq:cvar-bd-2}.
	
	In part (i) of the theorem above, we assumed that the density $f$ of $X$ is bounded below by $\frac{1}{\delta_1} > 0$. This implies that the derivative of \textrm{VaR} is bounded above. The latter condition is required for the trapezoidal rule to provide a good approx to the integral in \eqref{def:srm}. Moreover,
	the assumption that the first derivative of VaR w.r.t the confidence level $\beta$ is bounded implies that the underlying r.v. $X$ is bounded. 
	This claim can be made precise as follows:
	For any $\beta \in (0,1)$, it can be shown that (see Lemma \ref{lemma:var-derivatives} in the Appendix for a proof) 
	\begin{align}
	\mathrm{V}_\beta^{'} &=\frac{1}{f\left(\mathrm{V}_\beta\right)}, \textrm{ and }
	\mathrm{V}_\beta^{''} = -\frac{f^{'}\left(\mathrm{V}_\beta\right)}{{ f\left(\mathrm{V}_\beta\right)}^3} \label{eq:r067}
	\end{align}
	Notice that the first derivative of VaR involves a $1/f$ term, and  if the random variable is unbounded, then for every $\epsilon >0$, there is an $x$ such that $0<f(x)<\epsilon$. This implies $1/f$ cannot be bounded above uniformly w.r.t $x$, and hence the derivative of VaR cannot be bounded either.
	
	The stronger condition $\mid f^{'}(x)\mid/f(x)^3 \le \delta_2$ used in part (ii) of Theorem \ref{thm:srm-conc-bdd}, in conjunction with \eqref{eq:r067}, implies that the second derivative of \textrm{VaR} is bounded. Now, as before, a bounded second derivative implies that the underlying r.v. $X$ is bounded. To see this, the expression for the second derivative of VaR involves a $f'/f^3$ term, and if the r.v. $X$ is unbounded, then a uniform bound on $f'/f^3$ would mean that, as $x \rightarrow \infty$, $f$ decays too slowly to integrate to something finite, leading to a contradiction. More precisely, the differential inequality $|f'|/f^3 < K$ can be ``solved'' to get $f(x)>C/\sqrt{a+bx}$ for large $x$ and suitable constants $a$, $b$, and $C$. However, the expression on the RHS integrates to infinity, and hence, no density $f$ with unbounded support can have $f'/f^3$ bounded.
	
	\section{Gaussian and exponential distributions}
	\label{sec:srm-est-unbdd}
	\subsection{Estimation scheme}
	Let $X_1, \dots, X_n$ denote i.i.d. samples from the distribution of $X$. We form a truncated set of samples as follows:
	\begin{align*}
	\bar{X}_{i} = X_{i}\indic{X_{i}\le B_{n}}, 	\end{align*}  where  
	$B_{n}$ is a truncation threshold that depends on the underlying distribution.
	For the case of Gaussian distribution with mean zero and variance ${\sigma}^2$, we set $B_{n} = \sqrt{  2 \sigma^2\log{\left(n\right)}}$, and for the case of exponential distribution with mean $1/\lambda$, we set $B_{n} = \frac{\log(n)}{\lambda}$.
	
	We form an SRM estimate along the lines of \eqref{eq:srm-est}, except that the samples used are truncated samples, i.e., 
	\begin{equation}
	\label{eq:srm-est-trunc}
	\widetilde{\mathrm{S}}_{\,n,m} = \sum_{k = 1}^{m}\frac{\varphi(\beta_{k-1})\widetilde{\mathrm{V}}_{n,\beta_{k-1}} + \varphi(\beta_{k})\widetilde{\mathrm{V}}_{n,\beta_{k}}}{2}\Delta\beta.
	\end{equation}
	In the above, $\widetilde{\mathrm{V}}_{n,\beta} = \widetilde F_n^{-1}(\beta)$, with $\widetilde{F}_{n}(x) = \frac{1}{n}\sum_{i=1}^{n}\mathbb{I}[{\bar X_i \le x}]$.  
	
	\subsection{Concentration bounds}
	Next, we present  concentration bounds for our SRM estimator assuming that the samples are either from a Gaussian distribution with mean zero and variance $\sigma^2$, or from the exponential distribution with mean $1/\lambda$. Note that the estimation scheme is not provided this information about the underlying distribution. Instead $\widetilde{\textrm{S}}_{n,m}$ is formed from $n$ i.i.d. samples and with $m$ sub-intervals, using \eqref{eq:srm-est-trunc}. 
	
	\begin{theorem}[\textbf{SRM concentration: Gaussian case}]
		\label{thm:srm-conc-gauss}
		Assume (A1). Suppose that the r.v. $X$ is Gaussian with mean zero and variance $\sigma^2$, with $\sigma \le \sigma_{\textrm{max}}$. Fix $\epsilon>0$. 
		If $m \ge  
		\frac{1}{5}\sqrt{\frac{\sigma_{\textrm{max}}}{\epsilon}}\exp{\left(\frac{nc{\epsilon}^2}{4C^2_{1}}\right)}$, then
		\begin{align*}
		&\mathbb{P}\left[ \left| \mathrm{S} - {\widetilde{\mathrm{S}}}_{\,n,m} \right|>\epsilon\right]
		\le\frac{2\sigma\left(\sqrt{2\log{\left(n\right)}}\, C_{2}  + \sqrt{2\pi} n\, C_{1} \right)}{\left(\epsilon -\frac{2\sigma C_{1} }{\sqrt{n}}\right)}\exp{\left(-\frac{{nc{\left(\epsilon -\frac{2\sigma C_{1} }{\sqrt{n}}\right)}^2}}{{2C^2_{1}}}\right)}			
		, \forall \epsilon > \frac{2\sigma C_{1}}{\sqrt{n}}.
		\end{align*}
		where $c$ is as in Theorem \ref{thm:srm-conc-bdd} (i).  
	\end{theorem}
	\begin{proof}
		See Section \ref{sec:proof-srm-conc-gauss}.
	\end{proof}	
	
	\begin{theorem}[\textbf{SRM concentration: Exponential case}]
		\label{thm:srm-conc-exp}
		Assume (A1). Suppose that the r.v. $X$ is exponentially distribution with parameter $\lambda$, and $0 < \lambda_{\textrm{min}} \le \lambda $. Fix $\epsilon>0$.
		If $m \ge 
		\frac{1}{8}\sqrt{\frac{1}{\lambda_{\textrm{min}} \epsilon}}\exp{\left(\frac{nc{\epsilon}^2}{4C^2_{1}}\right)}
		$, then
		\begin{align*}
		&\mathbb{P}\left[ \left| \mathrm{S} - {\widetilde{\mathrm{S}}}_{\,n,m} \right|>\epsilon\right]
		\le
		\frac{2\left(\frac{log{\left(n\right)\, C_{2} }}{\lambda} + n\, C_{1}\right)}{\left(\epsilon -\frac{ C_{1} (n+1)}{\lambda n}\right)}\exp{\left(-\frac{{nc{\left(\epsilon -\frac{ C_{1} (n+1)}{\lambda n}\right)}^2}}{{2C^2_{1}}}\right)}, \forall \epsilon > \frac{ C_{1} (n+1)}{\lambda n}.
		\end{align*}
		where $c$ is as in Theorem \ref{thm:srm-conc-bdd} (i).
	\end{theorem}
	\begin{proof}
		See Section \ref{sec:proof-srm-conc-exp}.
	\end{proof}	
	\begin{remark} 
		Note that concentration bounds for CVaR estimation can be derived using a completely parallel argument to that of the proof of the theorems above, together with following choice for risk aversion function $\varphi(\beta) = 1/(1-\alpha) \indic{\beta>\alpha},\, \alpha \in (0,1)$. The CVaR-specific results are provided in Appendix \ref{sec:cvar-results}.
	\end{remark}
	
	\section{Simulation experiments}
	\label{sec:expts}
	In this section, we demonstrate the efficacy of our proposed method for SRM estimation \eqref{eq:srm-est}, which we shall refer to as SRM-Trapz. In our experiments, we set the risk aversion function as follows: $\varphi(\beta) = \frac{5\,e^{-5(1-\beta)}}{1-e^{-5}}, \beta \in [0,1] $. In the following sub-section, we consider a synthetic experimental setting to compare the accuracy of SRM estimators. Subsequently, we use SRM-Trapz as a subroutine in a vehicular traffic routing application (see section \ref{sec:expts_veh_rout}).
	
	\subsection{Synthetic setup}
	
	Figure \ref{fig: sampleVSerror} presents the estimation error as a function of the sample size for SRM-Trapz. The algorithm is run with two different sub-divisions. The samples are generated using a Gaussian distribution with mean $0.5$ and variance $25$. We observe that $\textrm{SRM-Trapz}$ with $500$ subdivisions performs on par with $\textrm{SRM-Trapz}$ with 150 subdivisions for every sample size. Further, as expected, increasing sample size leads to lower estimation error, while also increasing the confidence (demonstrated by the shrinkage in standard error). 
	
	\begin{figure}[h]
		\begin{minipage}{0.575\textwidth}
			\centerline{\includegraphics[width=\columnwidth]{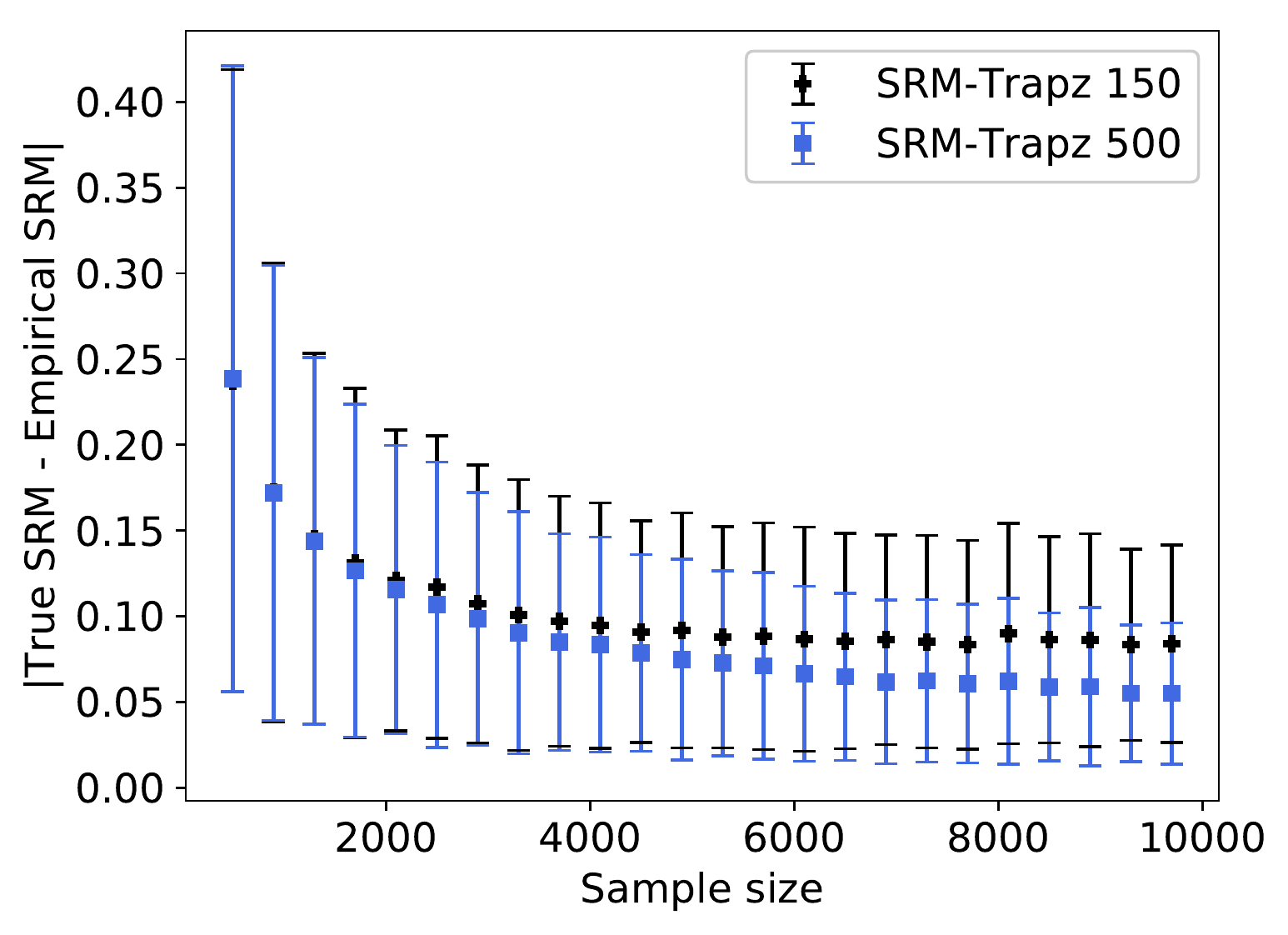}}
			\caption{Error in $\textrm{SRM}$ estimation ($\mid$True SRM - Empirical SRM$\mid$) on different sample size. \textrm{True SRM} is calculated using definition \ref{def:srm}. Empirical SRM is calculated by two methods, (\textbf{i}) $\textrm{SRM-Trapz}$ method with $m = 150$ subdivisions ($\textrm{SRM-Trapz}$ $150$), and (\textbf{ii}) $\textrm{SRM-Trapz}$ method with $m = 500$  subdivisions ($\textrm{SRM-Trapz}$ $500$). In both methods, $\textrm{SRM}$ is estimated using \eqref{eq:srm-est}. The underlying distribution considered for this simulation is $X \sim \mathcal { N } \left( 0.5 , 5 ^ { 2 } \right)$. The bars in the plot shows standard error averaged over $10^3$ iterations.}
			\label{fig: sampleVSerror}
		\end{minipage}
		\hspace{0.025\textwidth}
		\begin{minipage}{0.4\textwidth}
			\captionof{table}{The results for $\textrm{SRM}$ estimation, on four distributions, using two methods. Distributions are $\textbf{(a)}$ Exponential distribution with mean $1/0.2$ $(\text{Exp}(0.2))$, $\textbf{(b)}$ Normal distribution with mean zero and variance ${10}^2$ $(\mathcal{N}(0,10^2))$, $\textbf{(c)}$ Exponential distribution with mean $1/0.01$ $(\text{Exp}(0.01))$, $\textbf{(d)}$ Uniform distribution with range $-10^3$ to $10^3$ $(\text{U}(-{{10}^3}, {10}^3))$. Methods are
				$\textbf{(i)}$ Calculation of $\textrm{SRM}$ $(\textrm{SRM-True})$ using definition \ref{def:srm}, $\textbf{(ii)}$ $\textrm{SRM-Trapz}$ method with $m = 1000$ subdivisions ($\textrm{SRM-Trapz}$ $1000$) using \eqref{eq:srm-est}.
				In method (ii), $10^4$ i.i.d. samples are used for estimating $\textrm{SRM}$ on each distribution, and the standard error is averaged over $10^3$ iterations.}
			\label{Table: distributionVScvar}
			\vskip 0.15in
			\setlength{\tabcolsep}{0.05 cm} 
			\begin{tabular}{lccr}
				\toprule
				Distribution\,\, & SRM-True & SRM 1000 \\
				\midrule
				Exp(0.2)    &10.99 &11.02$\pm$1.21    \\
				N(0, $10^2$)    &107.36 &107.80$\pm$1.32   \\
				Exp(0.01)    &221.30 &221.39$\pm$2.47    \\
				U($-10^3, 10^3$) &612.47 &612.65$\pm$4.91  \\
				\bottomrule
			\end{tabular}
		\end{minipage}
	\end{figure}
	Table \ref{Table: distributionVScvar} presents the results obtained by SRM-Trapz with $1000$ subdivisions, for four different input distributions. We observe that $\textrm{SRM-Trapz}$ is comparable to SRM-True (calculated using definition \ref{def:srm}) under each input distribution.
	
	\subsection{Vehicular traffic routing}
	\label{sec:expts_veh_rout}
	In the vehicular routing application, the traditional objective is to find a route with the lowest expected delay. However, such an objective ignores risk factors. An alternative is to consider the weighted-sum delay of each route, and we use $\textrm{SRM}$ to quantify this objective. Thus, given a set of routes, the aim is to find the route (by adaptive sampling) with the lowest SRM of the delay. Simulation of Urban MObility (SUMO) \citep{behrisch2011sumo} is an open source, highly portable, microscopic road traffic simulation package designed to handle large road networks. Traffic Control Interface (TraCI) \citep{wegener2008traci} is a library, providing extensive commands to control the behavior of the simulation online, including vehicle state, road configuration, and traffic lights. We implement our routing algorithm using SUMO and TRACI.
	
	\begin{figure}[h]
		\begin{minipage}{0.5\textwidth}
			\begin{center}
				\centerline{\includegraphics[width=\columnwidth]{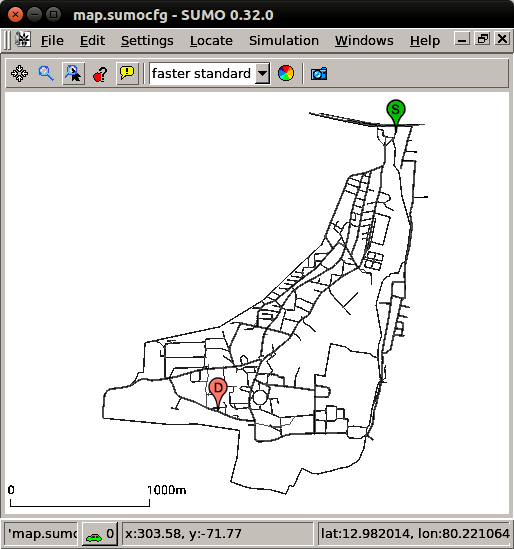}}
				\caption{Area of an urban city map, used for SUMO network.}
				\label{fig: network}
			\end{center}
		\end{minipage}
		\hspace{0.01\textwidth}
		\begin{minipage}{0.49\textwidth}
			\vskip -0.3in
			\begin{algorithm}[H]
				\caption{$\textrm{SRM-SR-Trapz}$}
				\label{alg: successive reject}
				\begin{algorithmic}
					\STATE {\bfseries Input:} number of rounds $n$, number of routes $K$, number of subdivisions $m$.
					\vskip 0.1in
					\STATE Let $A_{1} = \{ 1 , \ldots , K \}, \overline { \log } ( K ) = \frac { 1 } { 2 } + \sum _ { i = 2 } ^ { K } \frac { 1 } { i }, n _ { 0 } = 0 $ and for $ k \in \{ 1 , \ldots , K - 1 \}, n _ { k } = \left[ \frac { 1 } { \overline{\log} ( K ) } \frac { n - K } { K + 1 - k } \right]$
					\vskip 0.1in
					\STATE For each phase $ k = 1, 2, \dots, K - 1 :$
					\STATE (1) For each $ i \in A _ { k },$  select route  $i$  for $n _ { k } - n _ { k - 1 } $ rounds.
					\STATE (2) Let $ A _ { k + 1 } = A _ { k }\backslash\arg \max _ { i \in A _ { k } } \widehat{\textrm{S}}_{\, n _ { k },m,i}$ (we only remove one element from  $A _ { k }$ if there is a tie, select randomly the route to dismiss among the worst routes).
					\vskip 0.1in
					\STATE {\bfseries Output:} Let $i^*$ be the unique element of $A _ { K } $.
					\vskip 0.2in
					\STATE $\widehat{\textrm{S}}_{\,n _ { k },m,i}$ is \textrm{SRM} estimate for $i$th route, using \eqref{eq:srm-est} with $n_k$ samples, and m subdivisions.  
				\end{algorithmic}
			\end{algorithm}
		\end{minipage}
	\end{figure}
	
	For the experiments, we use the street map of the area around IIT Madras, Chennai, India (see Figure \ref{fig: network}) obtained from OpenStreetMap (OSM) \cite{haklay2008openstreetmap}, and then used Netconvert tool to load the map in SUMO. The network has 426 junctions and a total edge length of 123 km. We ran SUMO on this network for $30,000$ time-steps, in which  $7000$ cars, $500$ buses, $2000$ bikes, $1000$ cycles, and $1000$ pedestrians were added at different time-steps and in different lanes uniformly. We choose $K = 5$ routes between two fixed points, marked as S and D in Figure \ref{fig: network}. On these selected routes, we add $n = 1000$ cars and track them. In Table \ref{table: routes avg and srm}, $\widehat X_{n,i}$ is the estimated average delay of the $i$th route, and $\widehat{\textrm{S}}_{\,n,m,i}$ 
	is the SRM estimate for the $i$th route, $i = 1, \dots, K$, using \eqref{eq:srm-est}, and with $n$ samples. We set the number of subdivisions $m = 100$.
	
	\begin{table}[htb!]
		\vskip -0.1in
		\caption{Results for the estimated average delay ($\hat X_{n,i}$) and estimated SRM ($\widehat{\textrm{S}}_{\,n,m,i}$), for $i$th route, where $i = 1, \dots, K$.}
		\label{table: routes avg and srm}
		\begin{center}
			\begin{small}
				\begin{sc}
					\begin{tabular}{lcccccr}
						\toprule
						& route$_1$ & route$_2$ & route$_3$ & route$_4$ & route$_5$ \\
						\midrule
						$\hat X_{n,i}$ \hskip 5pt &283.81 &287.15 &306.80 &266.85 &325.86  \\
						$\widehat{\textrm{S}}_{\,n,m,i}$ &431.28 &361.81 &455.83 &378.68 &390.95 \\ 
						\bottomrule
					\end{tabular}
				\end{sc}
			\end{small}
		\end{center}
		\vskip -0.1in
	\end{table}
	
	From Table \ref{table: routes avg and srm} it is apparent that ROUTE$_4$ has the lowest expected delay, and ROUTE$_2$ has the lowest SRM. We consider a best-arm identification (BAI) bandit framework \cite{audibert2010best}, where an algorithm is given a fixed budget. Here, the budget refers to the total number of samples across routes. After the sampling budget, the algorithm is expected to recommend a route, and is judged by the probability that the recommended route is correct (i.e. the best route). 
	
	We ran successive rejects (SR), which is a popular BAI algorithm, except that SR is modified to find the route with lowest SRM. Note that the regular SR algorithm finds the route with the lowest expected delay.  Algorithm \ref{alg: successive reject} presents the pseudocode for the $\textrm{SRM-SR-Trapz}$ algorithm, with $\textrm{SRM-Trapz}$ used to form SRM estimates for each route. The setting of SUMO is as noted above. We set the budget $n = 1000$, number of routes  $K = 5$, and $m = 100$ subdivisions for $\textrm{SRM-Trapz}$. 
	We observed that Algorithm \ref{alg: successive reject} picks ROUTE$_2$ with probability $0.91$.
	
	\section{Convergence proofs}
	\label{sec:proofs}
	\subsection{Proof of Theorem \ref{thm:srm-conc-bdd}}
	\label{sec:proof-srm-conc-bdd}
	For establishing the bound in Theorem \ref{thm:srm-conc-bdd}, we require a result concerning the error of a trapezoidal-rule-based approximation, and a concentration bound for the $\mathrm{VaR}$ estimate  in \eqref{eq:var-est}. We state these results below, and subsequently provide a proof of Theorem \ref{thm:srm-conc-bdd}.
	\begin{lemma}\label{lemma:trap_error} 
		Let  $0 < a \le b < 1$, and $\{\beta_k\}_{k = 0}^{m}$ be a partition of $[a, b]$ such that $\beta_0 = a$ and $\beta_{k} = \beta_{k-1} + \Delta\beta$, $\Delta\beta = (b - a)/m$ is length of each sub-interval.
		
		\noindent \textbf{(i)} If $\left|\left(\varphi(\beta)V_\beta\right)^{\prime}\right| \le K_1 \,$  for $\beta \in [a,b]$,  then 
		\begin{align}
		\left| \int_a^b\varphi(\beta)\mathrm{V}_{\beta}\,d\beta - \sum_{k = 1}^{m}\frac{ \varphi(\beta_{k-1})\mathrm{V}_{\beta_{k-1}} +  \varphi(\beta_{k})\mathrm{V}_{\beta_{k}}}{2}\Delta\beta \right| \le \frac{K_1(b-a)^2}{4m}. \label{trap_err_1}
		\end{align}
		\noindent \textbf{(ii)} If $\left|\left(\varphi(\beta)V_\beta\right)^{\prime\prime}\right| \le K_2\,$ for $\beta \in [a,b]$, then
		\begin{align}
		\left| \int_a^b\varphi(\beta)\mathrm{V}_{\beta}\,d\beta - \sum_{k = 1}^{m}\frac{ \varphi(\beta_{k-1})\mathrm{V}_{\beta_{k-1}} +  \varphi(\beta_{k})\mathrm{V}_{\beta_{k}}}{2}\Delta\beta \right| \le \frac{K_2(b-a)^3}{12m^2}. \label{trap_err_2}
		\end{align}
		\begin{proof}
			See Appendix \ref{sec:proof_trap_error}.
		\end{proof}	
	\end{lemma}

	\begin{lemma}[\textbf{VaR concentration}]\label{Var_conc} 
		Let the r.v. X be continuous.
		Fix $\epsilon >0$, then we have 
		\begin{align*}
		\mathbb{P} \left[\left| \mathrm{V}_{\beta} - \widehat{\mathrm{V}}_{n,\beta}\right|\ge \epsilon\right] \le 2\exp\left(-2n\bar c{\epsilon}^2\right)
		\end{align*}
		where $\bar c$ is a constant that depends on the value of the density $f$ of the r.v. $X$ in a neighborhood of $\mathrm{V}_\beta$.
	\end{lemma}
	\begin{proof}[\textbf{Proof:}] See Proposition 2 in \citep{kolla2019concentration}.
	\end{proof}
	
	\begin{proof}[\textbf{Proof of Theorem \ref{thm:srm-conc-bdd}.}]

		First, we prove the claim in part (i). Notice that
		\begin{align*}
		&\mathbb{P}\left[ \left| \mathrm{S} - \widehat{\mathrm{S}}_{\,n,m} \right|>\epsilon\right]  \\
		&=\mathbb{P}\left[\left| \int_{0}^{1}\varphi(\beta)\mathrm{V}_{\beta}\,d\beta
		- \, \sum_{k=1}^m\frac{\varphi(\beta_{k-1})\widehat{\mathrm{V}}_{n,\beta_{k-1}} + \varphi(\beta_{k})\widehat{\mathrm{V}}_{n,\beta_k}}{2}\Delta \beta \right|>\epsilon\right]\\
		&=\mathbb{P}\left[\left| \int_{0}^{1}\varphi(\beta)\mathrm{V}_{\beta}\,d\beta
		-  \sum_{k=1}^m\frac{\varphi(\beta_{k-1}){\mathrm{V}}_{\beta_{k-1}} + \varphi(\beta_{k}){\mathrm{V}}_{\beta_k}}{2}\Delta\beta 
		\right. \right.\\ 
		&\left. \left. \quad + \sum_{k=1}^m\frac{\varphi(\beta_{k-1}){\mathrm{V}}_{\beta_{k-1}} + \varphi(\beta_{k}){\mathrm{V}}_{\beta_k}}{2}\Delta \beta 
		- \, \sum_{k=1}^m\frac{\varphi(\beta_{k-1})\widehat{\mathrm{V}}_{n,\beta_{k-1}} + \varphi(\beta_{k})\widehat{\mathrm{V}}_{n,\beta_k}}{2}\Delta \beta\right|>\epsilon\right]\\   
		&\le\mathbb{P}\left[\left|\sum_{k=1}^m\frac{\varphi(\beta_{k-1}){\mathrm{V}}_{\beta_{k-1}} + \varphi(\beta_{k}){\mathrm{V}}_{\beta_k}}{2}\Delta \beta
		- \, \sum_{k=1}^m\frac{\varphi(\beta_{k-1})\widehat{\mathrm{V}}_{n,\beta_{k-1}} + \varphi(\beta_{k})\widehat{\mathrm{V}}_{n,\beta_k}}{2}\Delta \beta \right|>\frac{\epsilon }{2}\right],
		\end{align*}
		where the final inequality follows by using Lemma \ref{lemma:trap_error}(i) to infer that for $m \ge \frac{K_1}{2\epsilon}$, we have 
		$ \left| \int_{0}^{1}\varphi(\beta)\mathrm{V}_{\beta}\,d\beta
		-\,  \sum_{k=1}^m\frac{\varphi(\beta_{k-1}){\mathrm{V}}_{\beta_{k-1}} + \varphi(\beta_{k}){\mathrm{V}}_{\beta_k}}{2}\Delta \beta \right|<\frac{\epsilon}{2}$.         
		Now, we have
		\begin{align*}
		&\mathbb{P}[ \mid S - \widehat{S}_{n,m} \mid>\epsilon]  \\
		&\le \mathbb{P}\left[\left|\sum_{k=1}^m\frac{\varphi(\beta_{k-1}){\mathrm{V}}_{\beta_{k-1}} + \varphi(\beta_{k}){\mathrm{V}}_{\beta_k}}{2}\Delta \beta
		- \, \sum_{k=1}^m\frac{\varphi(\beta_{k-1})\widehat{\mathrm{V}}_{n,\beta_{k-1}} + \varphi(\beta_{k})\widehat{\mathrm{V}}_{n,\beta_k}}{2}\Delta \beta \right|>\frac{\epsilon}{2}\right]\\
		&= \mathbb{P}\left[\left|\sum_{k=1}^m((\varphi(\beta_{k-1}){\mathrm{V}}_{\beta_{k-1}} + \varphi(\beta_{k}){\mathrm{V}}_{\beta_k})
		- \, (\varphi(\beta_{k-1})\widehat{\mathrm{V}}_{n,\beta_{k-1}} + \varphi(\beta_{k})\widehat{\mathrm{V}}_{n,\beta_k})) \right|>\frac{\epsilon}{\Delta\beta}\right]\\
		&= \mathbb{P}\left[\left| \varphi(\beta_{0})\mathrm{V}_{\beta_0} -\varphi(\beta_{0})\widehat{\mathrm{V}}_{n,\beta_0} + 2(\varphi(\beta_{1})\mathrm{V}_{\beta_1} -\varphi(\beta_{1})\widehat{ \mathrm{V}}_{n,\beta_1}) \right. \right. \\
		& \quad \left. \left.  + \dots + 2(\varphi(\beta_{m-1})\mathrm{V}_{\beta_{m-1}} -\varphi(\beta_{m-1})\widehat{ \mathrm{V}}_{n,\beta_{m-1}})
		+\, \varphi(\beta_{m})\mathrm{V}_{\beta_m} -\varphi(\beta_{m})\widehat{\mathrm{V}}_{n,\beta_m} \right|>\frac{\epsilon }{\Delta\beta}\right]\\
		&\le \mathbb{P}\left[\left|\varphi(\beta_{0})\mathrm{V}_{\beta_0} -\varphi(\beta_{0})\widehat{\mathrm{V}}_{n,\beta_0}\right|>\frac{\epsilon }{2m\,\Delta\beta}\right] 
		+ \, 2\mathbb{P}\left[\left|\varphi(\beta_{1})\mathrm{V}_{\beta_1} -\varphi(\beta_{1})\widehat{\mathrm{V}}_{n,\beta_1}\right|>\frac{\epsilon }{2m\,\Delta\beta}\right]\\
		& \quad 
		+ \dots + 2\mathbb{P}\left[\left|\varphi(\beta_{m-1})\mathrm{V}_{\beta_{m-1}} -\varphi(\beta_{m-1})\widehat{\mathrm{V}}_{n,\beta_{m-1}}\right|>\frac{\epsilon }{2m\,\Delta\beta}\right] \\
		& \quad
		+ \, \mathbb{P}\left[\left|\varphi(\beta_{m})\mathrm{V}_{\beta_m} -\varphi(\beta_{m})\widehat {\mathrm{V}}_{n,\beta_m}\right|>\frac{\epsilon }{2m\,\Delta\beta}\right]
		\end{align*}
		We now apply Lemma \ref{Var_conc} to bound each of the terms on the RHS above, to obtain
		\begin{align*}
		&\mathbb{P}\left[ \left| \mathrm{S} - \widehat{\mathrm{S}}_{\,n,m} \right|>\epsilon\right]\\
		&\le 2\exp\left(-2nc_{0}\left( \frac{\epsilon}{2\,m\varphi(\beta_{0}) \Delta\beta}\right)^2\right)
		+\,  4\exp\left(-2nc_{1}\left( \frac{\epsilon}{2\,m\varphi(\beta_{1}) \Delta\beta}\right)^2\right) \\
		& \quad
		+ \dots + 4\exp\left(-2nc_{m-1}\left( \frac{\epsilon}{2\,m \varphi(\beta_{m-1})\Delta\beta}\right)^2\right)
		+ 2\exp\left(-2nc_{m}\left( \frac{\epsilon}{2\,m \varphi(\beta_{m})\Delta\beta}\right)^2\right),
		\end{align*}
		where $c_i$ is a constant that depends on the value of the density $f$ in the neighborhood of $V_{\beta_i}$, for i = $0 \dots m$.
		Thus,
		\begin{align*}
		\mathbb{P}\left[ \left| \mathrm{S} - \widehat{\mathrm{S}}_{\,n,m} \right|>\epsilon\right]
		&\le 4m\exp\left(-2nc\left( \frac{\epsilon}{2\,m  C_{1}  \Delta\beta}\right)^2\right)\stepcounter{equation}\tag{\theequation}\label{eq:t122-m}\\
		&= 4m\exp\left( -\frac{n\,c\,\epsilon^2}{2C^2_{1}}\right) = \frac{2 K_1 }{\epsilon}\,\exp\left( -\frac{n\,c\,\epsilon^2}{2C^2_{1}}\right).
		\end{align*}
		Note that $c = \min\{c_0, c_1, \dots , c_m\}$ in \eqref{eq:t122-m}.
		The claim in part (i) follows.
		
		The proof of the result in part (ii) follows in a similar manner. In particular, 
		using part (ii) in Lemma \ref{lemma:trap_error}, with $m \ge \sqrt{\frac{K_2}{6\epsilon}}$, we obtain
		\begin{align*}
		\mathbb{P}[ \mid \mathrm{S} - \widehat{\mathrm{S}}_{\,n,m} \mid  >\epsilon]&\le 4m\exp\left( -\frac{n\,c\,\epsilon^2}{2C^2_{1}}\right)
		= \sqrt{8\,K_2/3\epsilon}\,.\exp\left( -\frac{n\,c\,\epsilon^2}{2C^2_{1}}\right)
		\end{align*}

	\end{proof}
	
	\subsection{Proof of Theorem \ref{thm:srm-conc-gauss}}
	\label{sec:proof-srm-conc-gauss}
	\begin{proof}[\textbf{Proof.}]
		Recall that the truncation threshold $B_{n} = \sqrt{2{\sigma}^2 \log{\left(n\right)}}$.  
		Letting $\eta = F(B_{n})$, we have 
		\begin{align*}
		&\mathbb{P}\left[  \mathrm{S} - {\widetilde{\mathrm{S}}}_{\,n,m}>\epsilon\right]\\
		&\le\mathbb{P}\left[ \int_{0}^{1}\varphi(\beta)\mathrm{V}_{\beta}\,d\beta
		- \sum_{k=1}^m\frac{\varphi(\beta_{k-1})\widetilde{\mathrm{V}}_{n,\beta_{k-1}} + \varphi(\beta_{k})\widetilde{\mathrm{V}}_{n,\beta_k}}{2}\Delta \beta > \epsilon\right]\\ 
		&=\mathbb{P}\left[ \int_{0}^{\eta}\varphi(\beta)\mathrm{V}_{\beta}\,d\beta
		- \sum_{k=1}^m\frac{\varphi(\beta_{k-1})\widetilde{\mathrm{V}}_{n,\beta_{k-1}} + \varphi(\beta_{k})\widetilde{\mathrm{V}}_{n,\beta_k}}{2}\Delta \beta
		+ \int_{\eta}^{1}\varphi(\beta)\mathrm{V}_{\beta}\,d\beta>\epsilon\right]\\
		&= \mathbb{P}\left[I_1 + I_2>\epsilon\right],\stepcounter{equation}\tag{\theequation}\label{eq:I1plusI2-m}
		\end{align*}
		where $I_{1} = \int_{0}^{\eta}\varphi(\beta)\mathrm{V}_{\beta}\,d\beta
		- \sum_{k=1}^m\frac{\varphi(\beta_{k-1})\widetilde{\mathrm{V}}_{n,\beta_{k-1}} + \varphi(\beta_{k})\widetilde{\mathrm{V}}_{n,\beta_k}}{2}\Delta \beta$, and $I_{2} = \int_{\eta}^{1}\varphi(\beta)\mathrm{V}_{\beta}\,d\beta$. \\[0.7ex]		
		We bound $I_2$ as follows:
		\begin{align}
		1-\beta = \mathbb{P}\left(X>{\mathrm{V}}_\beta\right)\le \exp{\left(-\frac{{{\mathrm{V}}_\beta}^2}{2\sigma^2}\right)},\label{eq:vbetabd}
		\end{align}
		since $X$ is Gaussian with mean zero, and variance $\sigma^2$. 
		Using $\log{x}\le\frac{x}{e}\,\,\forall{x>0}$, we obtain
		\begin{align*}
		{\mathrm{V}}_\beta \le \sqrt{2{\sigma}^2 \log{\left(\frac{1}{1-\beta}\right)}}    \le \sqrt{\frac{2{\sigma}^2}{e (1-\beta)}},
		\end{align*}
		leading to
		\begin{align*}
		\int^{1}_{\eta}{\mathrm{V}}_\beta\,d\beta &\le \sqrt{\frac{2{\sigma}^2}{e}}\int^{1}_{\eta}\frac{\,d\beta}{\sqrt{1-\beta}}
		= 2\sqrt{\frac{2{\sigma}^2}{e}}\sqrt{1-\eta} \\
		& \le 2\sqrt{\frac{2 {\sigma}^2}{e}}\exp{\left(-\frac{{{\mathrm{V}}_\eta}^2}{4{\sigma}^2}\right)} \tag{using \eqref{eq:vbetabd}}\\
		& = 2\sqrt{\frac{2 {\sigma}^2}{e}}\exp{\left(-\frac{{B_n}^2}{4{\sigma}^2}\right)} \tag{since $\mathrm{V}_\eta = B_n$}
		\end{align*}
		Hence,  
		\begin{align}
		I_2 = \int^{1}_{\eta}\varphi(\beta)\mathrm{V}_\beta \, d\beta \le C_1 \int^{1}_{\eta}\varphi(\beta)\mathrm{V}_\beta \, d\beta \le \frac{2\sigma C_{1} }{\sqrt{n}}.
		\label{eq:neps-constraint}
		\end{align}
		s		Applying the bound in the Theorem \ref{thm:srm-conc-bdd} to the truncated r.v. $Z = X\indic{X\le B_{n}}$, we bound $I_1$ as follows:	
		\begin{align*}
		\mathbb{P}\left[I_1>\epsilon\right]
		&\le\frac{{K}_{1}}{\epsilon}\exp{\left(-\frac{{nc{\epsilon}^2}}{{2C^2_{1}}}\right)} \stepcounter{equation}\tag{\theequation}\label{eq:r206}
		\end{align*}
		Hence,
		\begin{align*}
		\mathbb{P}\left[I_1+I_2>\epsilon\right]
		&\le\frac{{K}_{1}}{\left(\epsilon -\frac{2\sigma C_{1} }{\sqrt{n}}\right)}\exp{\left(-\frac{{nc{\left(\epsilon -\frac{2\sigma C_{1} }{\sqrt{n}}\right)}^2}}{{2C^2_{1}}}\right)} \tag{using \eqref{eq:neps-constraint} and \eqref{eq:r206}}\\
		&=\frac{\left(B_{n}\, C_{2}  + \delta_1\, C_{1} \right)}{\left(\epsilon -\frac{2\sigma C_{1} }{\sqrt{n}}\right)}\exp{\left(-\frac{{nc{\left(\epsilon -\frac{2\sigma C_{1} }{\sqrt{n}}\right)}^2}}{{2C^2_{1}}}\right)}\\
		&\le\frac{\left(\sqrt{2\sigma^2\log{\left(n\right)}}\, C_{2}  + \sqrt{2\pi}\sigma n\, C_{1} \right)}{\left(\epsilon -\frac{2\sigma C_{1} }{\sqrt{n}}\right)}\exp{\left(-\frac{{nc{\left(\epsilon -\frac{2\sigma C_{1} }{\sqrt{n}}\right)}^2}}{{2C^2_{1}}}\right)},
		\end{align*}
		where the final inequality follows from the fact that $\delta_1 = \sqrt{2\pi \sigma^2} \exp\left(\frac{B_n^2}{2\sigma^2}\right) = \sqrt{2\pi \sigma^2} n$, which holds since the underlying Gaussian distribution is truncated at $B_n$.
		
		By using a parallel argument, a concentration result for bounding the lower semi-deviations can be derived, and we omit the details.
	\end{proof}	
	\subsection{Proof of Theorem \ref{thm:srm-conc-exp}}
	\label{sec:proof-srm-conc-exp}
	\begin{proof}[\textbf{Proof.}]
		The proof for the exponential case follows in a similar manner as that of the proof of Theorem \ref{thm:srm-conc-gauss}. For the sake of completness, we provide the detailed proof in Appendix \ref{sec:proof-srm-conc-exp-appendix}. 
	\end{proof}
	
	\section{Conclusions}
	\label{sec:conclusions}
	We proposed a novel SRM estimation scheme that is based on numerical integration, and derived concentration bounds for our SRM estimator for the case of distributions with bounded support, Gaussian and exponential. As future work, it would be interesting to generalize the bounds for Gaussian/exponential distributions to the class of sub-Gaussian/sub-exponential distributions. An orthogonal direction for future work is to derive a lower bound for SRM estimation, and close the gap (if any) in the upper bound that we have derived.
	
	\bibliographystyle{plain}
	\bibliography{cvar_refs}
	
	\appendix
	\section{Proof of Lemma \ref{lemma:trap_error} }
	\label{sec:proof_trap_error}
	\begin{proof}[\textbf{Proof.}]
		The proof follows in a similar fashion as a result in \cite{cruz2003elementary}, and we provide the  details  below for the sake of completeness.
		Let $h = \Delta\beta = \frac{b-a}{m}$ and $\beta_k = a + kh$. We look at a single interval and operate integrate by parts twice:\\
		\begin{align}
		\int _ { \beta_ { k } } ^ { \beta _ { k  + 1} } \varphi(\beta)\mathrm{V}_\beta \, d \beta & = \int _ { 0 } ^ { h } \varphi(t+\beta_k)\mathrm{V}_ {\left( t + \beta _ { k } \right)}\, dt \notag\\
		&= \left[ ( t + A ) \varphi(t+\beta_k)\mathrm{V}_ {\left( t + \beta _ { k } \right)} \right] _ { 0 } ^ { h } - \int _ { 0 } ^ { h } ( t + A ) \left(\varphi(t+\beta_k)\mathrm{V}_ {\left( t + \beta _ { k } \right)}\right)^{\prime}\,d t \label{eq: trap_1_der}\\
		&= \left[ ( t + A ) \varphi(t+\beta_k)\mathrm{V}_ {\left( t + \beta _ { k } \right)} \right] _ { 0 } ^ { h } - \left[ \left( \frac { ( t + A ) ^ { 2 } } { 2 } + B \right) {\left(\varphi(t+\beta_k)\mathrm{V}_ {\left( t + \beta _ { k } \right)}\right)}^{\prime} \right] _ { 0 } ^ { h } \notag\\
		& \quad + \int _ { 0 } ^ { h } \left( \frac { ( t + A ) ^ { 2 } } { 2 } + B \right) {\left(\varphi(t+\beta_k)\mathrm{V}_ {\left( t + \beta _ { k } \right)} \right)}^{\prime \prime}dt. \label{eq: trap_2_der}
		\end{align}
		Setting $A = - h/2$ and $B=-h^2/8$ in the RHS above, we obtain 
		\begin{align*}
		\int _ { \beta _ { k } } ^ { \beta _ { k + 1}  } \varphi(\beta)\mathrm{V}_{\beta} \,d\beta & = \frac { h \left( \varphi(\beta_k)\mathrm{V}_{\beta_k} + \varphi(\beta_{k+1})\mathrm{V}_{\beta_{k+1}} \right) } { 2 }\\ 
		& \quad + \int _ { 0 } ^ { h } \left( \frac { ( t - h / 2 ) ^ { 2 } } { 2 } - \frac { h ^ { 2 } } { 8 } \right) \left(\varphi(t+\beta_k)\mathrm{V}_{(t+\beta_k)}\right)^{\prime \prime}\,dt
		\end{align*}
		Let $E_T(k)$ denote the difference between the integral above and the corresponding trapezoid. Then, the error in the trapezoidal rule approximation can be simplified as follows:
		\begin{align*}
		E _ { T } &= E _ { T } ( 0 ) + E _ { T } ( 1 ) + \cdots + E _ { T } ( m - 1 )\\
		&= \int _ { 0 } ^ { h } \left( \frac { ( t - h / 2 ) ^ { 2 } } { 2 } - h ^ { 2 } / 8 \right)  \left(\varphi(t+\beta_0)\mathrm{V}_{(t+\beta_0)}\right)^{\prime \prime}\,dt + \dots +\\
		& \quad \int _ { 0 } ^ { h } \left( \frac { ( t - h / 2 ) ^ { 2 } } { 2 } - h ^ { 2 } / 8 \right)  \left(\varphi(t+\beta_{m-1})\mathrm{V}_{(t+\beta_{m-1})}\right)^{\prime \prime}\,dt\\
		\begin{split}
		&= \int _ { 0 } ^ { h } \left( \frac { ( t - h / 2 ) ^ { 2 } } { 2 } - h ^ { 2 } / 8 \right)\left( \left(\varphi(t+\beta_0)\mathrm{V}_{\left( t + \beta _ { 0 } \right)}\right) ^ { \prime \prime } + \dots + \right.\\
		& \quad \left. \left(\varphi(t+\beta_{m-1})\mathrm{V}_{\left( t + \beta _ { m-1 } \right)}\right) ^ { \prime \prime } \right) d t
		\end{split}
		\end{align*}
		As in the text, we suppose that $\left| \left(\varphi(\beta)\mathrm{V}_\beta\right)^{\prime \prime} \right| \le K_2$ for $0\le \beta \le 1$. Then,
		\begin{align*}
		\begin{split}
		\left| E _ { T } \right| &= \left| \int _ { 0 } ^ { h } \left( \frac { {( t - {h }/ 2 )} ^ { 2 } } { 2 } - {h ^ { 2 }} / 8 \right) \left( \left(\varphi(t+\beta_0){\mathrm{V}_{\left( t + \beta _ { 0 } \right)}}\right) ^ { \prime \prime } + \dots +
		\right.\right. \\
		& \quad \left.\left.
		\left(\varphi(t+\beta_{m-1}){\mathrm{V}_{\left( t + \beta _ { m-1 } \right)}}\right)  ^ { \prime \prime } \right) d t \right|
		\end{split}\\
		\begin{split}
		& \le \int _ { 0 } ^ { h } \left| \left( \frac { ( t - h / 2 ) ^ { 2 } } { 2 } - h ^ { 2 } / 8 \right) \left( \left(\varphi(t+\beta_0){\mathrm{V}_{\left( t + \beta _ { 0 } \right)}}\right) ^ { \prime \prime } + \dots +
		\right.\right. \\
		& \quad \left.\left.
		\left(\varphi(t+\beta_{m-1}){\mathrm{V}_{\left( t + \beta _ { m-1 } \right)}}\right)  ^ { \prime \prime } \right) \right| d t
		\end{split}\\
		\begin{split}
		& = \int _ { 0 } ^ { h } \left|  \frac { {( t - h / 2 ) }^ { 2 } } { 2 } - {h ^ { 2 }} / 8 \right| \left| \left(\varphi(t+\beta_0){\mathrm{V}_{\left( t + \beta _ { 0 } \right)}}\right) ^ { \prime \prime } + \dots +
		\right. \\
		& \quad \left.
		\left(\varphi(t+\beta_{m-1}){\mathrm{V}_{\left( t + \beta _ { m-1 } \right)}}\right)  ^ { \prime \prime } \right| d t
		\end{split}\\
		\begin{split}
		& \leq \int _ { 0 } ^ { h } \left|  \frac { ( t - h / 2 ) ^ { 2 } } { 2 } - h ^ { 2 } / 8 \right| \left(\left|  \left(\varphi(t+\beta_0){\mathrm{V}_{\left( t + \beta _ { 0 } \right)}}\right) ^ { \prime \prime }\right| + \dots +
		\right. \\
		& \quad \left. \left| \left(\varphi(t+\beta_{m-1}){\mathrm{V}_{\left( t + \beta _ { m-1 } \right)}}\right) ^ { \prime \prime } \right| \right) d t
		\end{split}\\
		& \leq m K_2 \int _ { 0 } ^ { h } \left| \frac { ( t - h / 2 ) ^ { 2 } } { 2 } - \frac { h ^ { 2 } } { 8 } \right| d t.
		\end{align*}
		The function $\frac{\left(t-h/2\right)^2}{2} - \frac{h^2}{8}$ is a parabola opening upward that is zero at $t = 0$ and $t= h/2$. Thus, it is negative for $0 < t < h/2$ and positive elsewhere. Using this fact, we have
		\begin{align*}
		\int _ { 0 } ^ { h } \left| \frac { ( t - h / 2 ) ^ { 2 } } { 2 } - \frac { h ^ { 2 } } { 8 } \right| d t 
		&\leq \int _ { 0 } ^ { h/2 } \left|\left( \frac { ( t - h / 2 ) ^ { 2 } } { 2 } - \frac { h ^ { 2 } } { 8 } \right)\right| d t + \int _ { h/2 } ^ { h } \left|\left( \frac { ( t - h / 2 ) ^ { 2 } } { 2 } - \frac { h ^ { 2 } } { 8 } \right)\right| d t\\
		& = \left[ \left|\frac { ( t - h / 2 ) ^ { 3 } } { 6 } - \frac { h ^ { 2 } t } { 8 } \right|\right] _ { 0 } ^ { h/2 } + \left[ \left|\frac { ( t - h / 2 ) ^ { 3 } } { 6 } - \frac { h ^ { 2 } t } { 8 } \right|\right] _ { h/2 } ^ { h }\\
		& = \frac{h^3}{24}+\frac{h^3}{24} = \frac { h ^ { 3 } } { 12 }.
		\end{align*}
		Putting this all together and using $h = \frac{b-a}{m}$ gives us the following error bound:
		\begin{align}
		\left| E _ { T } \right| \leq \frac { m K_2 h ^ { 3 } } { 12 } = \frac { K_2(b-a)^3} { 12 m ^ { 2 } }.\label{eq:l377}
		\end{align}
		In the case when the second derivative of VaR is not bounded and instead, we have $\mid \left(\varphi(\beta)\mathrm{V}_\beta\right)^{\prime}\mid \le K_1$ for $\beta \in [a,b]$, one can employ an argument similar to that used above in arriving at \eqref{eq:l377}. In particular, starting with equation \eqref{eq: trap_1_der}, and constant of integration $A = h/2$, we obtain
		\begin{align*}
		\left| E _ { T } \right| 
		& \leq m K_1 \int _ { 0 } ^ { h } \left| t - h/2 \right| \,d t.
		\end{align*}
		The integral of function $\left| t - h/2 \right|$ is $h^2/4$.
		Putting it all together, and using $h = \frac{b-a}{m}$ leads to the following error bound:
		\begin{align*}
		\left| E _ { T } \right| \leq \frac { m K_1 h ^ { 2 } } { 4 } = \frac { K_1(b-a)^2} { 4 m }.
		\end{align*}
	\end{proof}
	
	\section{Derivative of VaR}
	We recall a result from \cite{dufour1995distribution} below.
	\begin{lemma}\label{lemma:var-derivatives} Let $F$ and $f$ are respectively probabilty distribution function and probability density function of continuous r.v. $X$, and assume (A1) holds in a neighbourhood of $\mathrm{V}_{\beta}$, where $0<\beta<1$, then $\mathrm{V_{\beta}}$ is twice differentiable.
		\begin{align*}
		\mathrm{V}_\beta^{'} &=\frac{1}{f\left(\mathrm{V}_\beta\right)}, \quad
		\mathrm{V}_\beta^{''} = -\frac{f^{'}\left(\mathrm{V}_\beta\right)}{{ f\left(\mathrm{V}_\beta\right)}^3}
		\end{align*}
	\end{lemma}
	
	\begin{proof}
		Notice that			$F(F^{-1}(\beta)) = \beta$, which implies  
		\begin{align}
		F^{\prime}(F^{-1}(\beta))F^{{-1}^{\prime}}(\beta) & = 1, \textrm{ and } 
		F^{''}(F^{-1}(\beta))(F^{-1'}(\beta))^2 
		+ \, F^{'}(F^{-1}(\beta))F^{-1''}(\beta) & = 0 \label{eq:second}
		\end{align}
		From  \eqref{eq:second}, we have
		\begin{align*}
		\mathrm{V}_\beta^{'} = F^{-1'}(\beta) &= \frac{1}{ F^{\prime}(F^{-1}(\beta))}=\frac{1}{ f(F^{-1}(\beta))}=\frac{1}{f\left(\mathrm{V}_\beta\right)}, \textrm{ and }\\
		\mathrm{V}_\beta^{''} = F^{-1''}(\beta) &= -\frac{ F^{''}(F^{-1}(\beta))(F^{-1'}{(\beta))}^2}{F^{'}(F^{-1}(\beta))}
		= -\frac{ f^{'}(F^{-1}(\beta)){(F^{-1'}(\beta))}^2}{f(F^{-1}(\beta))}\\
		&= - f^{'}(F^{-1}(\beta))(F^{-1'}(\beta))^3= -\frac{f^{'}\left(F^{-1}\left(\beta\right)\right)}{{ f\left(F^{-1}\left(\beta\right)\right)}^3} -\frac{f^{'}\left(\mathrm{V}_\beta\right)}{{ f\left(\mathrm{V}_\beta\right)}^3} 
		\end{align*}
	\end{proof}

	\section{Proof of Theorem \ref{thm:srm-conc-exp}}
	\label{sec:proof-srm-conc-exp-appendix}
	\begin{proof}[\textbf{Proof.}]
		
		The proof for the exponential case follows in a similar manner as that of the proof of Theorem \ref{thm:srm-conc-gauss}. In particular,  the proof up to  \eqref{eq:r206} holds for the exponential case, with a different bound on $I_2$. 
		
		We derive the bound on $I_2 =  \int_{\eta}^{1}\varphi(\beta)\mathrm{V}_{\beta}\,d\beta$. Using arguments similar to that in the Gaussian case, we obtain
		\begin{align*}
		\int^{1}_{\eta}{\mathrm{V}}_\beta\,d\beta &\le \frac{1}{\lambda}\int^{1}_{\eta}\log{\left(\frac{1}{1-\beta}\right)d\beta}\\
		& = \frac{(1-\eta)}{\lambda}\left(1+\log\left(\frac{1}{1-\eta}\right)\right)\\
		& \le \frac{(1-\eta)}{\lambda}\left(1+\frac{1}{(1-\eta)e}\right)\\
		& = \frac{\exp{\left(-\lambda {\mathrm{V}}_\eta\right)}}{\lambda}\left(1+\exp{\left(\lambda {\mathrm{V}}_\eta - 1\right)}\right)\\
		& = \frac{\exp{\left(-\lambda B_n\right)}}{\lambda} + \frac{1}{\lambda e} \tag{since $\mathrm{V}_\eta = B_n$}.
		\end{align*}
		Also,
		\begin{align*}
		\int_{\eta}^{1}\varphi(\beta)\mathrm{V}_{\beta}\,d\beta \le  C_{1} \int_{\eta}^{1}\mathrm{V}_{\beta}\,d\beta.
		\end{align*} 
		Choosing $B_{n} = \frac{\log(n)}{\lambda}$, we have
		\begin{align*}
		I_2 = \int^{1}_{\eta}\varphi(\beta)\mathrm{V}_\beta \, d\beta 
		\le\frac{ C_{1} (n+1)}{\lambda n} \stepcounter{equation}\tag{\theequation}\label{eq:neps_constraint_exp}
		\end{align*}
		Now, as in the proof of Theorem \ref{thm:srm-conc-gauss}, we have
		\begin{align*}
		\mathbb{P}\left[I_1>\epsilon\right]
		&\le\frac{{K}_{1}}{\epsilon}\exp{\left(-\frac{{nc{\epsilon}^2}}{{2C^2_{1}}}\right)}. \stepcounter{equation}\tag{\theequation}\label{eq:r208}
		\end{align*}
Thus,
		\begin{align*}
		\mathbb{P}\left[I_1+I_2>\epsilon\right]
		&\le\frac{{K}_{1}}{\left(\epsilon -\frac{ C_{1} (n+1)}{\lambda n}\right)}\exp{\left(-\frac{{nc{\left(\epsilon -\frac{ C_{1} (n+1)}{\lambda n}\right)}^2}}{{2C^2_{1}}}\right)}
		\tag{using \eqref{eq:neps_constraint_exp} and \eqref{eq:r208}}\\
		&=\frac{\left(B_{n}\, C_{2}  + \delta_1\, C_{1} \right)}{\left(\epsilon -\frac{ C_{1} (n+1)}{\lambda n}\right)}\exp{\left(-\frac{{nc{\left(\epsilon -\frac{ C_{1} (n+1)}{\lambda n}\right)}^2}}{{2C^2_{1}}}\right)}\\
		&\le\frac{\left(\frac{log{\left(n\right)\, C_{2} }}{\lambda} + n\, C_{1} \right)}{\left(\epsilon -\frac{ C_{1} (n+1)}{\lambda n}\right)}\exp{\left(-\frac{{nc{\left(\epsilon -\frac{ C_{1} (n+1)}{\lambda n}\right)}^2}}{{2C^2_{1}}}\right)}.
		\end{align*}
	where the final inequality follows from the fact that $\delta_1 = \exp(\lambda B_n) = n$, which holds since the underlying exponential distribution is truncated at $B_n$.
	
	By using a parallel argument, a concentration result for bounding the lower semi-deviations can be derived, and we omit the details.
	\end{proof}
	\section{CVaR estimation and concentration bounds}
	\label{sec:cvar-results}
	Acerbi's formula \citep{acerbi2002coherence}, an alternative form for $\mathrm{C_\alpha}(X)$ defined in \eqref{def:var-cvar}, is as follows: 
	\begin{align}
	\mathrm{C}_{\alpha}(X) = \frac{1}{1-\alpha}\int_{\alpha}^{1}\mathrm{V}_{\beta}(X)\,d\beta.\label{eq:acerbi}
	\end{align}
	From the expression above, $\mathrm{C}_{\alpha}(X)$ can be interpreted as the average of $\mathrm{V}_{\beta}(X)$ for $\beta \in [\alpha ,1)$.
	\subsection{Distributions with bounded support}
	\subsubsection{Estimation scheme}
	
	\label{sec:cvar-acerbi-est}

	Here, we propose to estimate $\mathrm{C}_\alpha(X)$, given n i.i.d. samples $X_1, \dots, X_n$ from the distribution of $X$, by approximating the integral in Acerbi's formula. 
	Notice that the integrand $V_\beta$ in \eqref{eq:acerbi} has to be estimated using the samples. 
	Let $\widehat{\mathrm{V}}_{n,\beta}$ denote the estimate of $\mathrm{V}_{\beta}(X)$, as given in \eqref{eq:var-est}. 
	We use the VaR estimates to form a discrete sum to approximate the integral in Acerbi's formula, an idea motivated by the trapezoidal rule \citep{cruz2003elementary}.
	More precisely, the estimate $\widehat{\mathrm{C}}_{n, m, \alpha}$ of $\mathrm{C}_\alpha(X)$ is formed as follows:
	\begin{equation}
	\label{eq:cvar-trapz-est}
	\widehat{\mathrm{C}}_{\,n,m,\alpha} = \frac{1}{1-\alpha}\sum_{k = 1}^{m}\frac{\widehat{\mathrm{V}}_{n,\beta_{k-1}} + \widehat{\mathrm{V}}_{n,\beta_{k}}}{2}\Delta\beta.
	\end{equation}
	In the above, $\{\beta_k\}_{k = 0}^{m}$ is a partition of $[\alpha, 1]$ such that $\beta_0=\alpha$ and $\beta_{k} = \beta_{k-1} + \Delta\beta$, where $\Delta\beta = (1 - \alpha)/m$ is the length of each sub-interval. 
	
	In the next section, we present concentration bounds for the estimator presented above, assuming that the underlying distribution has bounded support.
	
	
	\subsubsection{Concentration bounds}
	\label{sec:cvar-conc-bounds}
	For notational convenience, we shall use ${\mathrm{V}}_\alpha$ and ${\mathrm{C}}_\alpha$ to denote ${\mathrm{V}}_{\alpha}(X)$ and ${\mathrm{C}}_{\alpha}(X)$, for any $\alpha\in (0,1)$. 
	
	For all the results presented below, we take CVaR estimate as $\widehat{ \mathrm{C}}_{\,n,m,\alpha}$, $\,\,\alpha \in [0,1]$ be formed from $n$ i.i.d. samples of $X$ using \eqref{eq:cvar-trapz-est}. Let $F$ and $f$ denote the distribution and density of $X$, respectively.
	\begin{theorem}[\textbf{CVaR concentration: bounded case}]
		\label{thm:cvar-conc-bdd}
		Let the r.v. $X$ be continuous and $X \le B$ a.s.  Fix $\epsilon > 0$.
		
		\noindent \textbf{(i)} If $f(x) \ge \frac{1}{\delta_1} > 0 $,  $ \forall{x \in \left[F^{-1}(\alpha), B\right]}$, and $m \ge \frac{K_1\,(1-\alpha)}{2\epsilon}$, then 
		\begin{align*}
		&\mathbb{P}[ \mid \mathrm{C}_{\alpha} - \widehat{\mathrm{C}}_{\,n,m,\alpha} \mid>\epsilon] \le\frac{2K_1(1 - \alpha)}{\epsilon}\exp\left( \frac{-n\,c\,\epsilon^2}{2}\right),
		\end{align*}
		where $c = \min\{c_0, c_1, \dots , c_m\}$ and $c_k, k = \{0, \dots, m\}$ is a constant that depends on the value of the density $f$ of the r.v. $X$ in a neighborhood of $\textrm{V}_{\beta_{k}}$, with $\beta_k$ as in \eqref{eq:cvar-trapz-est}.

		\noindent \textbf{(ii)} If $\frac{\left|f^{'}(x)\right|}{f(x)^3} \le \frac{1}{\delta_2}$, $ \forall{x \in \left[F^{-1}(\alpha), B\right]}$  and $m \ge \sqrt{\frac{K_2\,{(1-\alpha)}^2}{6\epsilon}}$, then 
		\begin{align*}
		&\mathbb{P}[ \mid \mathrm{C}_{\alpha} - \widehat{\mathrm{C}}_{\,n,m,\alpha} \mid>\epsilon] \le \sqrt{\frac{8\,K_2 (1-\alpha)^2}{3\epsilon}}\exp\left( \frac{-n\,c\,\epsilon^2}{2}\right), \end{align*}
		where $c$ is as in the case above.
	\end{theorem}
	\begin{proof}
		Proceeds in a similar manner as the proof of Theorem \ref{thm:srm-conc-bdd}.
	\end{proof}

	\subsection{Gaussian and exponential distributions}
	\label{sec:cvar-est-unbdd}
	\subsubsection{Estimation scheme}
	Let $X_1, \dots, X_n$ denote i.i.d. samples from the distribution of $X$. We form a truncated set of samples as follows:
	\begin{align*}
	\bar{X}_{i} = X_{i}\indic{X_{i}\le B_{n}}, 	\end{align*}  where  
	$B_{n}$ is a truncation threshold that depends on the underlying distribution.
	For the case of Gaussian distribution with mean zero and variance ${\sigma}^2$, $B_{n} = \sqrt{  2 \sigma^2\log{\left(n\right)}}$, and for the case of exponential distribution with mean $1/\lambda$, $B_{n} = \frac{\log(n)}{\lambda}$.
	
	We form a CVaR estimate along the lines of \eqref{eq:cvar-trapz-est}, except that the samples used are truncated samples, i.e., 
	\begin{equation}
	\label{eq:cvar-trapz-est-trunc}
	\widetilde{\mathrm{C}}_{\,n,m,\alpha} = \frac{1}{1-\alpha}\sum_{k = 1}^{m}\frac{\widetilde{\mathrm{V}}_{n,\beta_{k-1}} + \widetilde{\mathrm{V}}_{n,\beta_{k}}}{2}\Delta\beta.
	\end{equation}
	In the above, $\widetilde{\mathrm{V}}_{n,\beta} = \widetilde F_n^{-1}(\beta)$, with $\widetilde{F}_{n}(x) = \frac{1}{n}\sum_{i=1}^{n}\mathbb{I}[{\bar X_i \le x}]$.  
	
	\subsubsection{Concentration bounds}
	Next, we present concentration bounds for our CVaR estimator assuming that the samples are either from a Gaussian distribution with mean zero and variance $\sigma^2$, or from the exponential distribution with mean $1/\lambda$. Note that the estimation scheme is not provided this information about the underlying distribution. Instead $\widetilde{\textrm{C}}_{n,m,\alpha}$ is formed from $n$ i.i.d. samples and with m sub-intervals, using \eqref{eq:cvar-trapz-est-trunc}. 
	
	\begin{theorem}[\textbf{CVaR concentration: Gaussian case}]
		Suppose that the r.v. $X$ is Gaussian with mean zero and variance $\sigma^2 >0$, with $\sigma \le \sigma_{\text{max}}$. Fix $\epsilon>0$.
		If $m \ge  
		\frac{1}{5}\sqrt{\frac{\sigma_{\text{max}}(1-\alpha)}{\epsilon}}\exp{\left(\frac{nc{\epsilon}^2}{4}\right)}$, then 
		\begin{align*}
		&\mathbb{P}\left[ \left| \mathrm{C} - {\widetilde{\mathrm{C}}}_{\,n,m,\alpha} \right|>\epsilon\right]
		\le\frac{2(1-\alpha)\sigma\sqrt{2\pi} n\, }{\left(\epsilon -\frac{2\sigma }{(1-\alpha)\sqrt{n}}\right)}\exp{\left(-\frac{{nc{\left(\epsilon -\frac{2\sigma }{(1-\alpha)\sqrt{n}}\right)}^2}}{{2}}\right)}			
		, \textrm{ for } \epsilon > \frac{2\sigma }{(1-\alpha)\sqrt{n}}.
		\end{align*}
		where $c$ is as in Theorem \ref{thm:cvar-conc-bdd} (i).  
	\end{theorem}
	\begin{proof}[\textbf{Proof.}]
		Proceeds in a similar manner as the proof of Theorem \ref{thm:srm-conc-gauss}.
			\end{proof}
	
	\begin{theorem}[\textbf{CVaR concentration: Exponential case}]
		Assume r.v. $X\sim\textrm{Exp}\left(\lambda\right)$ and $0 < \lambda_{\text{min}} \le \lambda $. Fix $\epsilon>0$.
		If $m \ge 
		\frac{1}{8}\sqrt{\frac{(1-\alpha)}{\lambda_{\text{min}} \epsilon}}\exp{\left(\frac{nc{\epsilon}^2}{4}\right)}
		$, then 
		we have 
		\begin{align*}
		&\mathbb{P}\left[ \left| \mathrm{C} - {\widetilde{\mathrm{C}}}_{\,n,m,\alpha} \right|>\epsilon\right]
		\le
		\frac{2(1-\alpha)n}{\left(\epsilon -\frac{ (n+1)}{(1-\alpha)\lambda n}\right)}\exp{\left(-\frac{{nc{\left(\epsilon -\frac{ (n+1)}{(1-\alpha)\lambda n}\right)}^2}}{{2}}\right)}, \textrm{ for } \epsilon > \frac{(n+1)}{(1-\alpha)\lambda n}.
		\end{align*}
		where $c$ is as in Theorem \ref{thm:cvar-conc-bdd} (i).
	\end{theorem}
	\begin{proof}[\textbf{Proof.}]
		Proceeds in a similar manner as the proof of Theorem \ref{thm:srm-conc-exp}.
	\end{proof}
	
	
\end{document}